\documentclass[letterpaper, 10 pt, conference]{ieeeconf}  % Comment this line out if you need a4paper

\IEEEoverridecommandlockouts                              % This command is only needed if 
 \usepackage{caption}
\usepackage{mathtools}
\usepackage{nomencl}
\usepackage{cite}
\usepackage{ntheorem}   % for theorems
\usepackage{lipsum}     % for sample text
\usepackage{graphicx}
\usepackage{algpseudocode}

\usepackage{algorithm}
\usepackage{algpascal}
\usepackage{esint}
\usepackage[latin9]{luainputenc}
\usepackage{amssymb}
\usepackage{lineno,hyperref}
\usepackage{xcolor}
\modulolinenumbers[5]
\usepackage{bm}
\usepackage{amsmath}

\DeclareRobustCommand{\uvec}[1]{{%
		\ifcsname uvec#1\endcsname
		\csname uvec#1\endcsname
		\else
		\bm{\mathbf{#1}}%
		\fi
}}

\theoremstyle{plain}

\newtheorem{prop}{Proposition}
\newtheorem{thm}{Theorem}
\theoremstyle{definition}
\newtheorem{example}{Example}

 \theoremstyle{remark}
 \newtheorem{rem}{Remark}
 \theoremstyle{normal}
\title{\LARGE \bf
Singularity-Free Inverse Dynamics for Underactuated Systems with a Rotating Mass
}

\author{Seyed Amir Tafrishi$^{1}$, Mikhail Svinin$^{2}$ and Motoji Yamamoto$^{1}$% <-this % Mstops a space
	\thanks{$^{1}$ Seyed Amir Tafrishi and Motoji Yamamoto are with the Department of Mechanical Engineering, Kyushu University, Kyushu, Japan {\tt\small amir@ce.mech.kyushu-u.ac.jp} and {\tt\small yama@mech.kyushu-u.ac.jp}}%
	\thanks{$^{2}$ Mikhail Svinin is with the Department of Information Science and Engineering, Ritsumeikan University, Kyoto, Japan \break	{\tt\small svinin@fc.ritsumei.ac.jp}}%
}

\renewcommand{\baselinestretch}{0.86} %.86
\begin{document}

\maketitle

%%%%%%%%%%%%%%%%%%%%%%%%%%%%%%%%%%%%%%%%%%%%%%%%%%%%%%%%%%%%%%%%%%%%%%%%%%%%%%%%
\begin{abstract}
	Motion control of underactuated systems through the inverse dynamics contains configuration singularities. These limitations in configuration space mainly stem from the inertial coupling that passive joints/bodies create. In this study, we present a model that is free from singularity while the trajectory of the rotating mass has a small-amplitude sine wave around its circle. First, we derive the modified non-linear dynamics for a rolling system. Also, the singularity regions for this underactuated system is demonstrated. Then, the wave parameters are designed under certain conditions to remove the coupling singularities. We obtain these conditions from the positive definiteness of the inertia matrix in the inverse dynamics. Finally, the simulation results are confirmed by using a prescribed Beta function on the specified states of the rolling carrier. Because our algebraic method is integrated into the non-linear dynamics, the proposed solution has a great potential to be extended to the Lagrangian mechanics with multiple degrees-of-freedom.
\end{abstract}

%%%%%%%%%%%%%%%%%%%%%%%%%%%%%%%%%%%%%%%%%%%%%%%%%%%%%%%%%%%%%%%%%%%%%%%%%%%%%%%%
\section{INTRODUCTION}

\label{intro}
Controlling the mechanisms of the machines and robots require an accurate mathematical model of the system. This model should contain all the physical characteristics of the system while it is computationally efficient. However, the control of the underactuated systems with passive bodies have certain challenges that can originate from the derived model \cite{liu2013survey}. Physically, the underactuated systems [see Fig. \ref{Fig:generality}] consist of two main parts: First, a rotating mass that moves by an actuator. Second, a passive body that displaces depending on the rotational mass.

Early studies for underactuated mechanisms have begun by the introduced Pendubot \cite{spong1995pendubot} and Acrobat \cite{spong1995swing} as the two-link manipulators. The general motion equations of a passive and an active rotating bodies with rotational angles of $\bm{q}=[\theta,\gamma]^T$ can be presented as
\begin{align}
\uvec{M}(\bm{q})\bm{\ddot{q}}+\uvec{h}(\bm{q},\bm{\dot{q}})=\uvec{u},
\label{Eq:MotionEquationGeneral}
%\label{Eq:motionmechanismsystem}
\end{align}
where inertial matrix $\uvec{M}(\bm{q})$, velocity dependencies $N_i$ and gravity terms $G_i$ in $\uvec{h}(\bm{q},\bm{\dot{q}})$ and control inputs $\uvec{u}$ are defined by
\begin{align*}{\small
\uvec{M}(\bm{q})= \left[\begin{array}{cc} M_{11}\;  M_{12}\\ M_{21}\; M_{22}\\ \end{array}\right],\uvec{h}(\bm{q},\bm{\dot{q}})= \left[\begin{array}{c}N_1+G_1\\ N_2+G_2\\ \end{array}\right],\uvec{u}=\left[\begin{array}{c}0\\ \tau_{\gamma}\\ \end{array}\right].}
\end{align*}
The underactuated systems (\ref{Eq:MotionEquationGeneral}) with two degrees of freedom \cite{liu2013survey} have a great common, similar inertial matrix $\uvec{M}(\bm{q})$, with certain underactuated spherical robots \cite{kayacan2012modeling,svinin2015dynamic,TafrishiASME2019,TAfrishiRussiCha2019}. This inertial similarity help us to generalize our studying problem. The rolling spherical robots propel their passive carrier with a rotating mass-point \cite{ilin2017dynamics,TafrishiASME2019,TAfrishiRussiCha2019} or pendulum \cite{kayacan2012modeling,svinin2015dynamic} as Fig. \ref{Fig:generality}.

From the control point-of-view, Agrawal in 1991 \cite{agrawal1991inertia} found that when the inertial matrix of the inverse dynamics is singular at certain configurations, the integration of differential equations breaks. Arai and Tachi \cite{arai1991position} showed that the inverse dynamics in a two-degree-of-freedom underactuated manipulator hit the singularity when coupled inertia, $M_{21}$ term in the first constraint equation imposed by the passive body in (\ref{Eq:MotionEquationGeneral}), becomes zero and this property limits the domain of the control. Spong \cite{spong1994partial} proposed a Strong Inertial Coupling condition for these underactuated systems under the positive definiteness of the inertia matrix. The condition grants a singular free inverse dynamics under certain geometric properties. In other words, these mathematical singularities that originate from the inversed terms of the inertial matrix $\uvec{M}(\uvec{q})$, limit the mechanism to certain geometric parameters and create a challenge in manipulation around these singularity regions. Furthermore, a coupling index was proposed to determine the actuability of underactuated systems with different geometries \cite{bergerman1995dynamic}. Later studies took place by following these coupling conditions \cite{spong1994partial,bergerman1995dynamic} for controlling the Pendubot \cite{zhang2003hybrid} and Acrobat \cite{spong1995swing,spong1996energy}. The same problem was highlighted and control strategies are developed relative to this limitation for spherical robots \cite{svinin2015dynamic,TAfrishiRussiCha2019}. Also, because the spherical carrier requires consecutive rotations without any angular limitations, the singularities due to inertial coupling become more challenging to deal with.

\begin{figure}[t!]
	\centering
	\includegraphics[width=2.5 in]{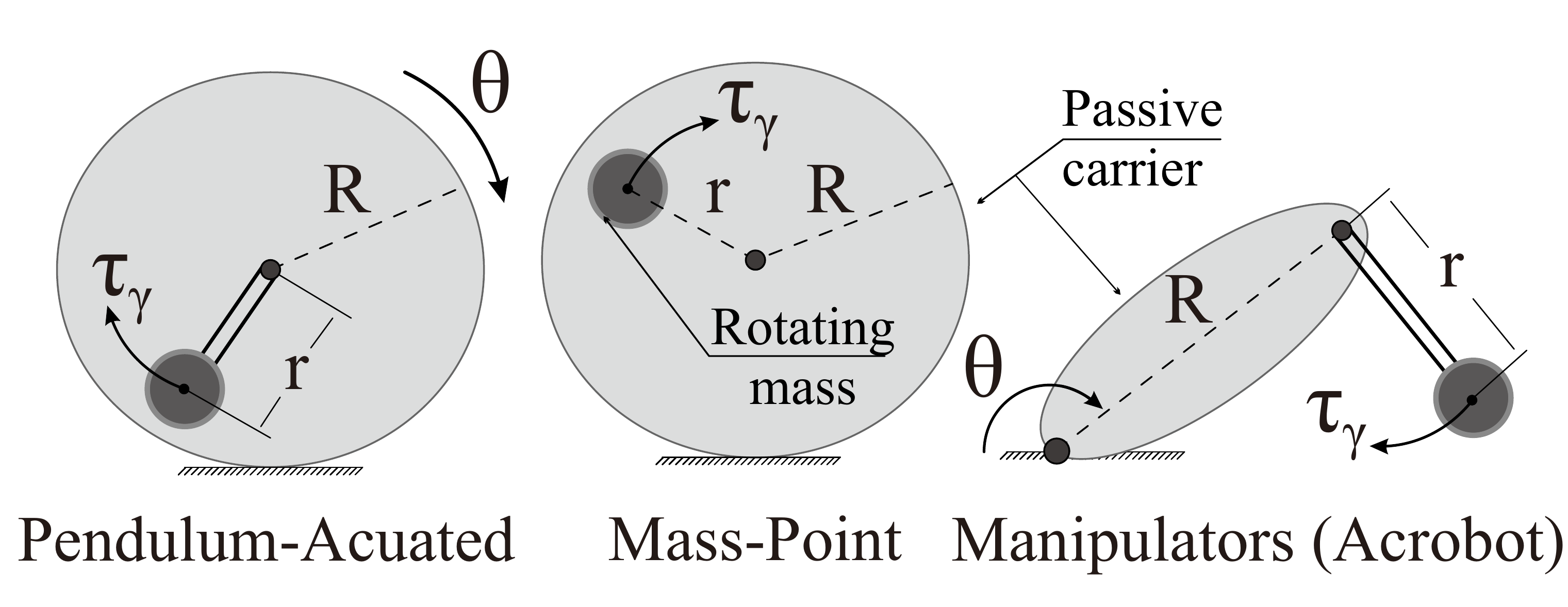}		
	\caption{Different mass-rotating systems with passive body/carrier that their inertial matrices $\uvec{M}(\bm{q})$ are similar.}\label{Fig:generality}
\end{figure}

%In early researches, the energy-based controllers were developed to swing up the system toward the unstable position \cite{spong1995swing,spong1996energy,zhang2003hybrid}. 

%The stabilizable and reachable areas for two-link manipulator were shown experimentally and analytically by Yabuno \cite{yabuno2005reachable} which there was no state feedback to the system. (No need because it checks unstable region (bifurcation) by a normilized paramteres.)

Our motivation in this paper is to propose a modeling approach  for the underactuated system in which the limitations on the geometric parameterization and configuration singularities due to inertial coupling \cite{spong1994partial,bergerman1995dynamic} are avoided. Thus, there aren't any physical design limitations due to mathematical singularities. This modified model is free from any complex algorithm. Also, we check the singularity regions in conventional rolling systems for the first time. In this work, we begin by defining the sine wave with small amplitude that is combined around the circular rotation of mass. Next, the new kinematic model is applied to the Lagrangian equations for finding the rolling spherical carrier's modified dynamics. Then, the inverse dynamics are derived and singularity regions are analytically demonstrated for this rolling system. Finally, a theory between wave and singularity regions is developed. This condition helps to design our proposed wave proportional to the considered physical system. Furthermore, the theoretical findings are verified for a rolling carrier with an actuating mass-point by the simulations. Also, a classic model of the pendulum-actuated system is compared with our modified model to clarify the validity of the theory in simulation space.

For the rest of this paper, we show the kinematics of the wavy trajectory for the rotating mass and also derive the non-linear dynamics in Section II. In Section III, by finding the inverse dynamics, the singularity-free conditions for the obtained model are explained. Finally, Section IV shows example simulations for a mass-point system with obtained singularity-free conditions and compares the modified model with the classical one.

\section{MODIFIED DYNAMICAL MODEL}
\label{DynamicalModel}
In this section, we introduce a sinusoidal trajectory that is combined around a circle for the rotating mass. Next, the developed kinematics is substituted into the Lagrangian function of a rolling system. Finally, the Lagrangian method is utilized to find the nonlinear dynamics of this underactuated model. Using this property, we will propose a theory that the singularities due to inertial coupling are removed through designing the included wave.
\subsection{Trajectory with a Combined Wave}
\label{WavesCircle}
 Let us assume that the rotating mass has an orientation angle of $\gamma$ with respect to the center of the spherical carrier with a radius of $R$ [see Fig. \ref{Fig:wavy2}]. Also, the carrier is rolling with an angle of $\theta$. Then, the position vector of the mass that rotates around a small-amplitude sinusoidal curve on the circle with a radius of $r$ is defined as  
\begin{align}
\begin{split}
&\uvec{D}_c= -\left[\left(r+a\sin \left( n (\gamma+\theta)+\varepsilon\right)\right)\cos(\gamma+\theta) \right]\uvec{k} \\
&- \left[ \left(r+a\sin \left( n (\gamma+\theta) +\varepsilon\right)\right) \sin (\gamma+\theta) \right] \uvec{j},
\end{split}
\label{Eq:Sinsiodal}
\end{align}
where $a$, $n$ and $\varepsilon$ are the amplitude of sinusoidal wave, the frequency of created periodic wave on the circle of radius $r$ and constant phase shift of the curve, respectively. In the classic mass-rotating models, this trajectory becomes 
\begin{align}
\uvec{D}_w= -r(\cos(\gamma+\theta)\uvec{k} + \sin (\gamma+\theta) \uvec{j})
\label{Eq:classicrotatingmasstrajectry}
\end{align} 
where $a$, $n$ and $\varepsilon$ are zero in (\ref{Eq:Sinsiodal}) that gives a circular rotation with radius $r$ \cite{TafrishiASME2019,TAfrishiARM2019}. 
\begin{figure}[t!]
	\centering
	\includegraphics[width=1.6 in,height=1.25 in]{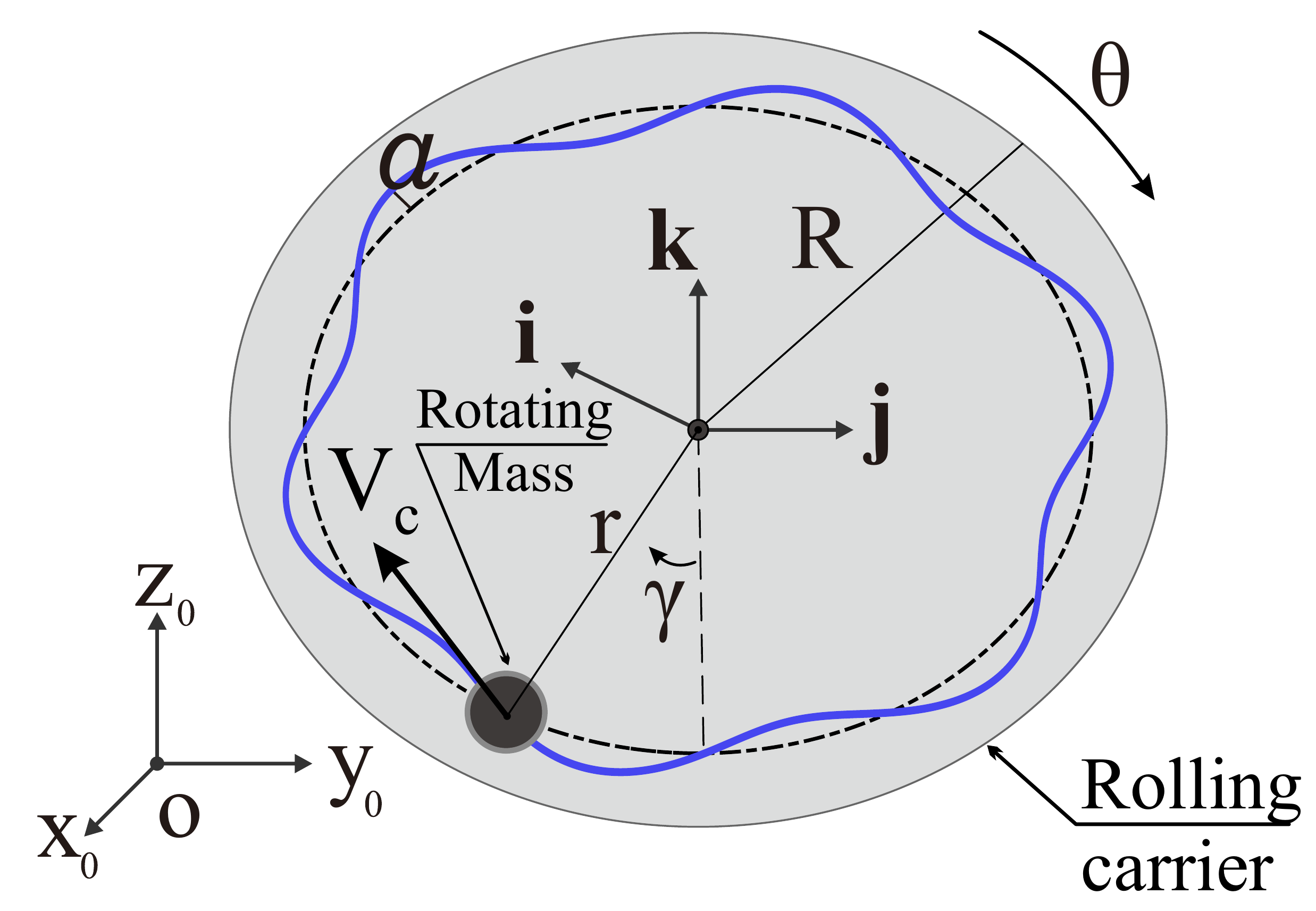}	
		\includegraphics[width=1.6 in]{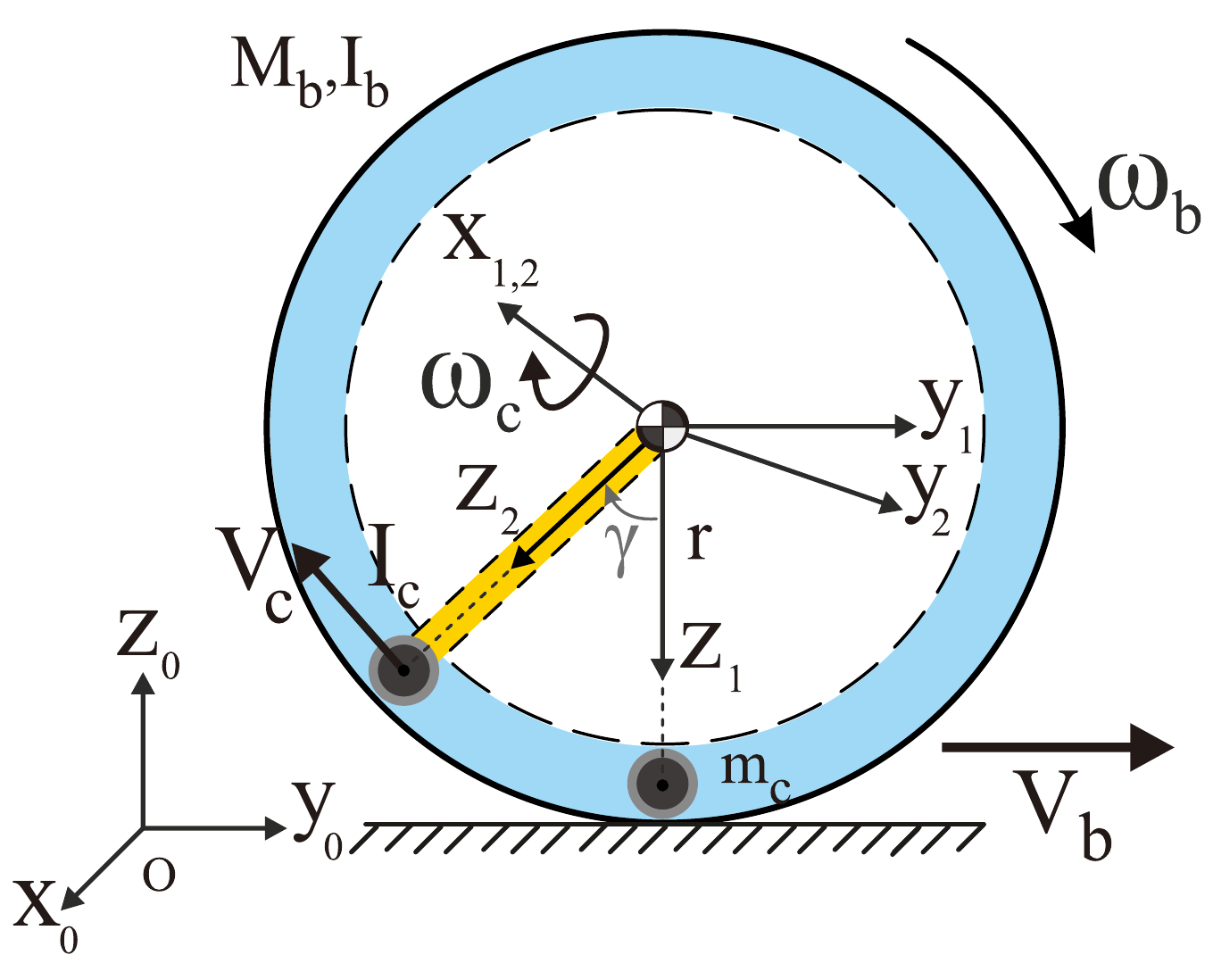}
      \hspace{30 pt} (a) \hspace{120 pt} (b)
	\caption{a) A rotating mass with the trajectory of a sine wave (blue color) on a circle with radius $r$ (black dashed line). b) Rolling carrier motion along y axis and frame transformations. %Also, an arbitrary inertia $I_c$ respect to sphere center stem from either the yellow pendulum model or blue air/fluid medium that pushes the mass-point.
	}\label{Fig:wavy2}
\end{figure}  We aim to design $n$, $a$ and $\varepsilon$ depending on the obtained relations from the inertial matrix to removes the inertial coupling singularity \footnote{Please check Theorem \ref{prop:singularfreeparametersd} for details.\label{footi}} while rotating mass follows around wavy circle. Also, the deviation of trajectory $\uvec{D}_c$ with respect to circular radius $r$ can be found as 
\begin{align}
\Delta D_c= r-\parallel \uvec{D_c}\parallel  =a\sin(n(\gamma+\theta)+\varepsilon) , 
\label{Eq:deviationTrajectory}
\end{align}
where $\parallel\cdot\parallel$ stands for the module of the variable. Maximum value for this deviation $\Delta D_c$ will be the amplitude of wave $a$, which we assume always $a \ll r$ and $a \ll R$. If the property ($a \ll r,R$) is satisfied, the effect of included sine wave can be ignored in dynamic models. However, the larger wave amplitudes $a$ can be also realized in driving mechanisms. For instance, a prismatic joint can move the mass periodically on the lead of the pendulum-actuated systems \cite{stilling2002controlling,gandino2014damping} or fluid-actuated systems can have a pipe in corresponding wave form.

To find the rolling kinematics, the coordinate frames are sketched as Fig. \ref{Fig:wavy2}-b. Here, $x_0 y_0 z_0$ represents the reference frame. The moving frame connected to the center of the spherical carrier is $x_1y_1z_1$, which
translates with respect to reference frame $x_0y_0z_0$. Finally, $x_2y_2z_2$ is a rotating frame for the rotating mass-point attached to the center of spherical carrier and it is rotating with respect to
$x_1y_1z_1$. The corresponding kinematics for a rolling carrier with rotational mass is 
 \begin{align}
 \bm\omega_b=\dot{\theta}\uvec{i},\uvec{V}_b=R\dot{\theta}\uvec{j},\; \bm\omega_{c}=\left(\dot{\gamma}+\dot{\theta}\right)\uvec{i},\;\uvec{V}_{c}=\uvec{V}_b+\dot{\uvec{D}}_c
 \label{geomtrycharaccorevel}
 \end{align} 
 where $\bm\omega_b$, $\uvec{V}_b$, $\bm\omega_{c}$ and $\uvec{V}_{c}$ are the angular and linear velocities of the carrier and the angular and linear velocities of the rotating mass.  
 
\subsection{Nonlinear Dynamics}
The non-linear dynamics of the rolling spherical carrier with a planar motion is derived from the proposed trajectory equation. To find the corresponding motion equations, the Lagrangian equations are utilized.

We consider a sphere as a passive carrier (passive joint) where it is actuated with the rotation of a spherical mass as Fig. \ref{Fig:wavy2}. The carrier has a mass of $M_b$ excluding the rotating mass. Also, the rotating mass with the mass $m_c$ is assumed as a mass-point. The Lagrangian function of the rolling carrier with the rotating mass, including kinetic and potential energies, along $y$ axis is described \cite{TAfrishiARM2019} as follows
\begin{align}
\begin{split}
&E_L=\dfrac{1}{2}M_b\parallel \uvec{V}_b \parallel ^2+\dfrac{1}{2}I_b\parallel \bm\omega_b\parallel^2+\dfrac{1}{2}m_c\parallel\uvec{V}_{c}\parallel^2\\
&+\dfrac{1}{2}I_c\parallel  \bm\omega_c\parallel ^2-m_c gd_{c},
\end{split}
\label{EQ:EnergyEquationtotalforces}
\end{align}
where $I_b=2 M_bR^2 / 3$, $I_c$, $g$ and $d_c$ are the inertia tensor of rolling passive carrier, an arbitrary inertia tensor $I_c$ connected to the mass-point, the acceleration of gravity and the distance of the mass-point respect to the ground, respectively. We include the inertia tensor $I_c$ for the sake of generality that its rotation is with the respect to carrier central frame $x_1y_1z_1$. This arbitrary inertia tensor $I_c$ can be considered as either the lead of rotating pendulum \cite{kayacan2012modeling,svinin2015dynamic} (yellow pendulum in Fig. \ref{Fig:wavy2}) or interacting fluid/gas inside pipes for the rotating spherical mass \cite{TAfrishiRussiCha2019}  (blue fluid/gas in Fig. \ref{Fig:wavy2}). After the substitution of Eq. (\ref{geomtrycharaccorevel}) into (\ref{EQ:EnergyEquationtotalforces}), one obtains
\begin{align}
\begin{split}
&E_L=\dfrac{1}{2}R^2\dot \theta^2M_{b}+\dfrac{1}{2}I_b\dot \theta^2+\dfrac{1}{2}I_c(\dot\gamma+\dot \theta)^2+\dfrac{1}{2}m_c\Big[\big [R\dot\theta\\
&- \left(\dot\gamma+\dot\theta \right)\left(a n \cos \left(n (\gamma+\theta) + \varepsilon \right) \right)\sin(\gamma+\theta)\\
& +  ( r+a \sin (n (\gamma+\theta) +\varepsilon) ) \cos (\gamma+\theta) \big ]^2 \\
&+ \left(\dot\gamma+\dot\theta \right)^2 \big [ a n \cos (n (\gamma+\theta) +\varepsilon) \cos (\gamma+\theta) \\
& -\left(r+a \sin \left( n (\gamma+\theta) + \varepsilon \right) \right)\sin(\gamma+\theta) \big ]^2 \Big]\\
&- m_c g [r+a\sin (n (\gamma+\theta)+ \varepsilon) ] (1-\cos(\gamma+\theta))\\
\end{split}
\label{Eq:lastlagrangian}
\end{align}

Finally, we apply the Lagrangian equations for planar translation along y axis as following 
\begin{align}
\dfrac{d}{dt}\left(\dfrac{\partial E_L}{\partial \dot \gamma}\right)- \dfrac{\partial E_L}{\partial \gamma}= \tau_\gamma,\;\dfrac{d}{dt}\left(\dfrac{\partial E_L}{\partial \dot \theta}\right)- \dfrac{\partial E_L}{\partial \theta}= \tau_\theta,
\label{Eq:LagrangianEquationmain}
\end{align}
where $\tau_{\gamma}$ and $\tau_{\theta}$ are the external torques for the rotating mass and the sphere. Acting external torque between the surfaces of the spherical mass and carrier body is assumed zero, $\tau_{\theta}=0$, since mass-point doesn't contain any spinning around itself and it only rotates with the respect to the carrier center $x_1y_1z_1$. After doing the necessary operations by Eqs. (\ref{Eq:lastlagrangian})-(\ref{Eq:LagrangianEquationmain}), the terms of the equations of the motion (\ref{Eq:MotionEquationGeneral}) for this underactuated system becomes
%\begin{align}
%\left[\begin{array}{ccc}
%M_{11} & M_{12} \\
%M_{21} & M_{22}\\
%\end{array}\right]
%\left[\begin{array}{ccc}
%\ddot{\theta} \\
%\ddot{\gamma}\\
%\end{array}\right]+\left[\begin{array}{ccc} N_{1}\\ N_{2}\\ \end{array}\right]+\left[\begin{array}{ccc} G_{1}\\ G_{2}\\ \end{array}\right]
%=\left[\begin{array}{ccc} 0 \\ \tau_\gamma\\ \end{array}\right], 
%\label{Eq:motionmechanismsystem}
%\end{align}
%where $M_{ij}$, $N_{i}$ and $G_{i}$ are the inertia coefficients, the velocity dependencies and the gravity factors as
\begin{align}
\begin{split}
&M_{11}=I_c+M_b R^2+I_b+m_c R^2-2m_c R \mu_{1} + m_c \mu_2, \\
&M_{12}=M_{21}=I_c -m_c R \mu_1 + m_c \mu_2,M_{22}=I_c+ m_c \mu_2,\\
&N_{1}=-m_c R (\dot\gamma+\dot\theta)^2 \mu_3 + m_c (\dot\gamma+\dot\theta)^2 \mu_4 ,\\
&N_{2 }=   m_c  (\dot\gamma+\dot\theta)^2 \mu_4 , G_{1}= G_{2}=m_c   g \mu_5.
\end{split}
\label{Eq:NONLINEARDynamics} 
\end{align}
while,
\begin{align}
\begin{split}
&\mu_1= \left(a n \cos \left(n (\gamma+\theta) + \varepsilon \right) \right)\sin(\gamma+\theta)\\
& +  ( r+a \sin (n (\gamma+\theta) +\varepsilon) ) \cos (\gamma+\theta),  \\
&\mu_2= a^2 n^2 \cos^2 \left(n (\gamma+\theta)+ \varepsilon \right)+ (r+a\sin(n(\gamma+\theta)+ \varepsilon ))^2, \\
&\mu_3= -a n^2 \sin \left(n (\gamma+\theta)+ \varepsilon \right) \sin (\gamma+\theta) \\
& + 2 a n \cos(n(\gamma+\theta) +\varepsilon) \cos(\gamma+\theta)\\
& -(r+a\sin(n(\gamma+\theta) +\varepsilon))\sin(\gamma+\theta),  \\
&\mu_4= -a^2 n^3 \sin \left(n (\gamma+\theta)+ \varepsilon \right) \cos (n(\gamma+\theta) +\varepsilon) \\
& + a n \cos(n(\gamma+\theta ) +\varepsilon) (r+a\sin(n(\gamma+\theta)+\varepsilon)) \\
&\mu_5=  a n \cos(n(\gamma+\theta)+\varepsilon) (1-\cos(\gamma+\theta))\\
&+ ( r+a \sin (n (\gamma+\theta) +\varepsilon) ) \sin(\gamma+\theta).   
\end{split}
\label{Eq:termsofmusindirect}
\end{align}                      
 
\section{INVERSE DYNAMICS AND SINGULARITY}
In this section, the non-linear dynamics are presented in inverse form. A general condition for removing the singularity is derived and the singularity regions are analyzed for classic rolling systems. Next, we propose our theory for determining the parameters of the combined wave to avoid the singularity regions that originate from the coupled inertia matrix. Finally, a Beta function as feed-forward control for specifying the spherical carrier rotation is given.

The non-linear dynamics (\ref{Eq:MotionEquationGeneral}) with Eq. (\ref{Eq:NONLINEARDynamics}) are re-ordered with the goal to find the input torque $\tau_{\gamma}$ from the specified rolling carrier states ($\theta, \; \dot\theta,\; \ddot\theta$). Hence, the rolling constraint of the carrier and the rotating mass differential equations in Eq. (\ref{Eq:MotionEquationGeneral}) becomes
\begin{align}
\begin{split}
&\ddot{\gamma}=-\frac{1}{M_{12}}\left( M_{11} \ddot{\theta}+N_1+G_1\right),\\
&\tau_{\gamma}=   M_{21} \ddot{\theta}+ M_{22} \ddot{\gamma}+N_2+G_2.
\end{split}
\label{Eq:InversefirstEQ}
\end{align}
Now, we knew from (\ref{Eq:MotionEquationGeneral}) that inertial matrix $\uvec{M}(q)$ is always a positive definite and symmetric matrix \cite{agrawal1991inertia} where upper-left determinants grant this condition by $M_{11}>0$ and $M_{11}M_{22}-M_{12}M_{21}>0$. To extend these conditions to the derived inverse dynamics, the rolling constraint (first differential equation) in Eq. (\ref{Eq:InversefirstEQ}) is substituted into the second differential equation as follows
\begin{align}
\overline{\tau}_{\gamma}= \overline{M} \ddot{\theta}+\overline{N} +\overline{G},
\label{Eq:FinalEqInverse}
\end{align}
where 
\begin{align}
\begin{split}
&\overline{\tau}_{\gamma}=-\tau_{\gamma},\;
\overline{M}= M^{-1}_{12}\cdot \left(M_{11}M_{22}-M_{12}M_{21} \right),\;\\
&\overline{N}= M_{22}M^{-1}_{12}N_1-N_2 ,\;
\overline{G}= M_{22}M^{-1}_{12}G_1-G_2.
\label{Eq:ExtraFinalTerms}
\end{split}
\end{align}
Because the mass-point and the carrier rotation are opposite of each other in our motion and for the sake of the simplicity, we assume $\overline{\tau}_{\gamma}=-\tau_{\gamma}$. By relying on the Ref.\cite{spong1994partial}, the coupled inertia matrix $\overline{M}>0$ should be positive definite as well. However, denominator in $\overline{M}>0$ requires another extra condition that $M^{-1}_{12}>0$. Under the condition of $M_{11}M_{22}-M_{12}M_{21}>0$,  there exist singularities in the solution of Eq. (\ref{Eq:FinalEqInverse}) for the cases when $M_{12}\leq0$ ($\tau_{\gamma}\rightarrow \infty$) \cite{agrawal1991inertia}. Thus, following proposition as the condition of the singularity is expressed.
\begin{prop}
	Let the inverse non-linear dynamics (\ref{Eq:FinalEqInverse})-(\ref{Eq:ExtraFinalTerms}) are for the rolling system with the trajectory $\uvec{D}_c$ in (\ref{Eq:Sinsiodal}). Given $M_{12}>0$, the underactuated system does not hit any singularity, if following condition is satisfied
	\begin{align}
	\mu^2_{1a}+\mu^2_{1b}+\frac{I_c}{m_c}>R\left[\mu_{1a}\sin(\gamma+\theta)
	+\mu_{1b}\cos(\gamma+\theta)\right],
	\label{Eq:condtionsingularity}
	\end{align}
	where $\mu_{1a}$ and $\mu_{1b}$ are the first and second terms of $\mu_1$.
	\label{Propos:Singularitycondition}
\end{prop}
\begin{proof}
	Consider the inertia term $M_{12}$ in (\ref{Eq:ExtraFinalTerms}) always positive to the avoid singularity 
	\begin{align}
	M_{12}= I_c -m_c R \mu_1 + m_c \mu_2 >0.
	\label{Eq:Condtiotnsingularitymc}
	\end{align}
	Then, $\mu_{1a}$ and $\mu_{1b}$ terms are defined from $\mu_1$ in (\ref{Eq:termsofmusindirect}) as 
	\begin{align*}
	\begin{split}
	&\mu_{1a}= a n \cos \left(n (\gamma+\theta) + \varepsilon \right) \\
	& \mu_{1b}=  r+a \sin (n (\gamma+\theta) +\varepsilon)  .
	\end{split}
	\end{align*}
	where there are $\mu_2=\mu_{1a}^2+\mu_{1b}^2$ and $\mu_1=\mu_{1a}\sin(\gamma+\theta)+\mu_{1b}\cos(\gamma+\theta)$. Then, $\mu_{1a}$ and $\mu_{1b}$ are substituted to inequality (\ref{Eq:Condtiotnsingularitymc}) as follows
	\begin{align}
	\begin{split}
	&m_c(\mu^2_{1a}+\mu^2_{1b})-m_cR(\mu_{1a}\sin(\gamma+\theta)\\
	&+\mu_{1b}\cos(\gamma+\theta))+I_c>0
	\end{split}
	\label{Eq:Thefinalinequality}
	\end{align}
	Finally, the condition (\ref{Eq:condtionsingularity}) is found by reordering inequality (\ref{Eq:Thefinalinequality}).
\end{proof}
 
Before designing our combined wave model under the Proposition \ref{Propos:Singularitycondition}, we check the singularity regions for the different classical underactuated rolling systems. Note that in these cases the trajectory is considered as an ideal circle (\ref{Eq:classicrotatingmasstrajectry}) without any consideration of our combined sinusoidal curve. 
\begin{example}
	Singularity region of a classical rotating mass-point system \cite{TAfrishiARM2019,TafrishiASME2019}, where $I_c=0$, are analyzed using condition (\ref{Eq:condtionsingularity}) in Proposition \ref{Propos:Singularitycondition}. Let the trajectory $\uvec{D}_c$ be a perfect circle with radius $r$, which makes $\mu_{1a}=0$ and $\mu_{1b}=r$ as (\ref{Eq:classicrotatingmasstrajectry}) when $a=n=\varepsilon=0$. From the given condition (\ref{Eq:condtionsingularity}), the inequality is transformed to 
	\begin{align}
	r^2 >Rr\cos(\gamma+\theta)
	\label{Eq:Theconditionmass-pintsin}
	\end{align}
	\begin{figure}[t!]
		\centering
		\includegraphics[width=2.6 in]{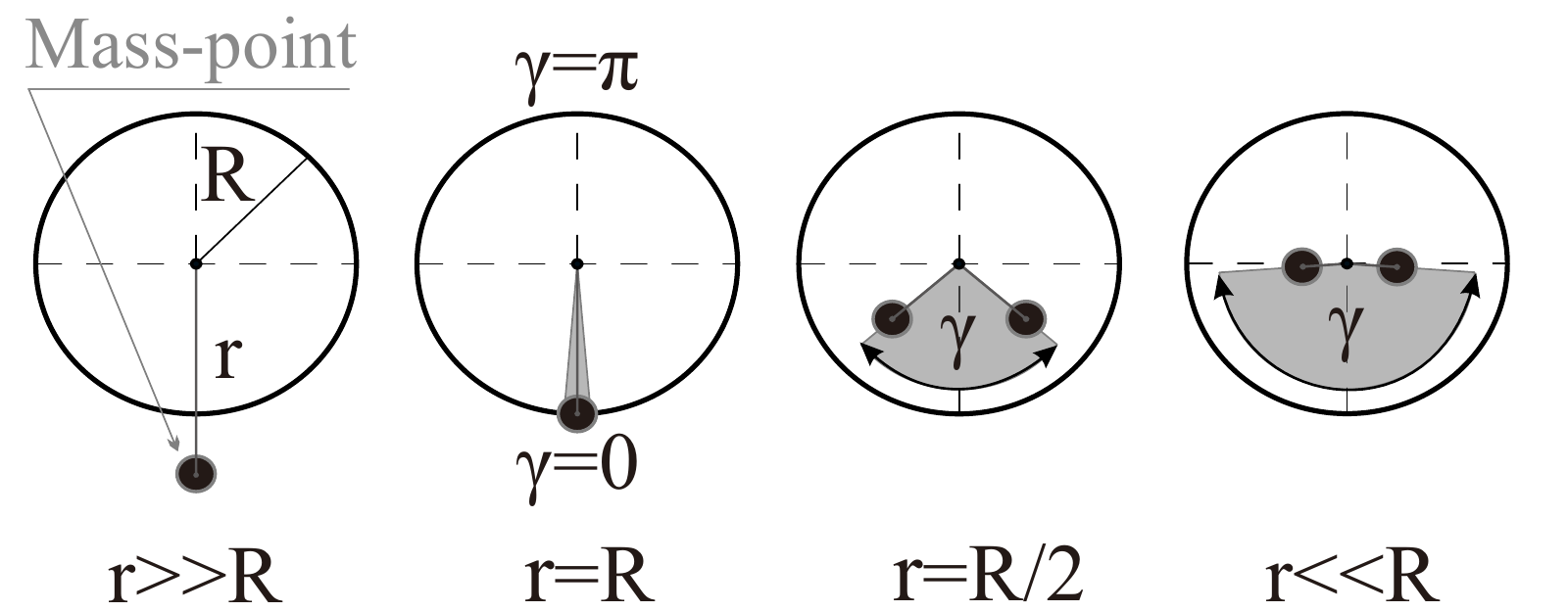}	
		\caption{Singularity regions in gray color increases as mass-point distance $r$ decreases while the carrier is steady ($\theta=0$).}\label{Fig:Singualarity_region}
	\end{figure}
	Now, by considering the maximum possible value for $\cos(\gamma+\theta)\approx1$, one obtains a limitation on the geometric parametrization as 
	\begin{align}
	\frac{r}{R}>1.
	\label{Eq:distancemasspointexample}
	\end{align}
    This means in designing this mass-point system, the inverse dynamics model (\ref{Eq:InversefirstEQ}) will hit singularity if the radius of rotating mass be less than rolling carrier as condition (\ref{Eq:distancemasspointexample}) and singularity disobeys the physical mechanics completely. Fig. \ref{Fig:Singualarity_region} shows how changes in the geometric parameters ($r,R$) in (\ref{Eq:Theconditionmass-pintsin}) affect singular configurations of the mass-point (\ref{Eq:classicrotatingmasstrajectry}) on the steady spherical carrier ($\theta=0$). This graphic clarifies that rolling systems without any angular constraint on $\gamma$ and $\theta$ will hit the singularity. Otherwise, the solution of (\ref{Eq:InversefirstEQ}) will break many times while this singular region changes by spherical carrier rotation, $\gamma+\theta$.
\end{example}

\begin{example}
    A rotating mass system with arbitrary inertial tensor is chosen in this example. This inertia tensor $I_c$ can be related to a rod that connects the mass to the center of the rolling body \cite{kayacan2012modeling,svinin2015dynamic} or an interacted water with the rotating mass in pipes \cite{TAfrishiRussiCha2019,TafrishiASME2019}. Thus, with a circular trajectory as the previous example, condition (\ref{Eq:distancemasspointexample}) is transformed to
	\begin{align}
	m_cr^2+I_c>m_cRr \cos(\gamma+\theta),
	\label{Eq:CouplingconditioninIc}
	\end{align}
	By sorting this condition based on the $I_c$ with considering $\cos(\gamma+\theta)\approx1$, we see that singularity can be avoided only when 
	\begin{align}
	I_c>m_cr(R-r).
	\label{Eq:SingularFreePendl}
	\end{align} 
    Similar to the previous example, singularity limits the inverse dynamics for only certain mechanisms that can satisfy the following geometric condition.
\end{example}

To show that our proposed approach removes the demonstrated singularity regions in Eq. (\ref{Eq:Theconditionmass-pintsin}) and Eq. (\ref{Eq:CouplingconditioninIc}), and how parameters of the included wave should be designed analytically a theory is developed.
\begin{thm}
	\label{prop:singularfreeparametersd}
	The inverse dynamics (\ref{Eq:FinalEqInverse}) with a combined sinusoidal wave (\ref{Eq:Sinsiodal}) never hit singularity and positive definiteness of $\overline{M}>0$ is granted when variables $a$, $n$ and $\varepsilon$ of the small-amplitude wave ($a \ll r, R$ and $n>2$) are satisfying following inequalities
	\begin{equation} {\small
	\begin{split} 
	&  r^2+\frac{a^2}{2}\left [n^2+(n^2-1) \cos{2\varepsilon}+1 \right]	+ \frac{I_c}{m_c} > \Delta \mu_{1} \\ 
	&  r^2+\frac{a^2}{2}\left [n^2+(n^2-1) \cos 2\varepsilon+1 \right]	+ \frac{I_c}{m_c}> \frac{\sqrt{2}}{2}  R r + \Delta \mu_{2} \\
	&   r^2+\frac{a^2}{2}\left [n^2+(n^2-1) \cos 2\varepsilon+1 \right]	+ \frac{I_c}{m_c} > R r + \Delta \mu_{3}
	\end{split}
	\label{Eq:ConditionsForsingularitDesign}}
	\end{equation}	
	where 
	\begin{equation*}
	{\footnotesize
	\begin{split}
	& \Delta \mu_1=a \cdot \Bigg|2r\sin\left(\tan^{-1}\left(\frac{-2r}{Rn}\right)\right)-Rn\cos\left(\tan^{-1}\left(\frac{-2r}{Rn}\right)\right)\Bigg|,\\
	&\Delta \mu_2=\frac{\sqrt{2}a}{2}\bigg|(2\sqrt{2}r-R)\sin\left(\tan^{-1}\left(\left(R-2\sqrt{2}r\right)/Rn\right)\right)\\
	&-Rn\cos\left(\tan^{-1}\left(\left(R-2\sqrt{2}r\right)/Rn\right)\right)\bigg|,\\
	&\Delta \mu_3=a\cdot |2r-R|.
	\end{split}}
	\end{equation*}
\end{thm}
\begin{proof}
	Let the singularity condition (\ref{Eq:condtionsingularity}) from Proposition \ref{Propos:Singularitycondition} be
	\begin{align*}
	\mu^2_{1a}+\mu^2_{1b}+(I_c/m_c)>R(\mu_{1a}\sin(\gamma+\theta)
	+\mu_{1b}\cos(\gamma+\theta)).
	\end{align*}
  	To have the left-hand side of the inequality always larger than the right-side, we find the absolute value of right-side in the three cases
	\begin{equation}{\small
	\begin{split}
	&1)\; \mu^2_{1a}+\mu^2_{1b}+\frac{I_c}{m_c}>R\cdot|\mu_{1a}|, \;\;\;\;\;\;\;\;\;\;\;\;\;\;\;\;\;\;\;\;\; \; \zeta_1=\frac{(2k+1)\pi}{2}\\
	&2)\;\mu^2_{1a}+\mu^2_{1b}+\frac{I_c}{m_c}>\frac{\sqrt{2}R}{2}\cdot\left(|\mu_{1a}|+|\mu_{1b}|\right),\; \zeta_2=\frac{(2k+1)\pi}{4}\\
	&3)\;\mu^2_{1a}+\mu^2_{1b}+\frac{I_c}{m_c}>R\cdot|\mu_{1b}|, \;\;\;\;\;\;\;\;\;\;\;\;\;\;\;\;\;\;\;\;\; \;\zeta_3=(k+1)\pi
	\end{split} }
	\label{Eq:InequalitiesNEw3side}
	\end{equation}
	where $\zeta_i=\gamma+\theta$. Next, we utilize the Fourier Transform equations \cite{vretblad2003fourier} as follows
	\begin{align}
	\begin{split}
	&H(w)=\frac{1}{\sqrt{2\pi}} \int_{-\infty}^{\infty}\mu(\zeta_i) e^{-j w \zeta_i } d\zeta_i, \\
	&\mu(\zeta_i)=\frac{1}{\sqrt{2\pi}} \int_{-\infty}^{\infty}H(w) e^{j w \zeta_i } dw, 
	\end{split}	
	\label{Eq:FourierEquation}
	\end{align}
	where $H(w)$ and $w$ are the transformed term of $\mu$ and the frequency of corresponding $\mu$. 
	With applying Fourier Transform (\ref{Eq:FourierEquation}) to each side of inequalities in (\ref{Eq:InequalitiesNEw3side}), under linearity property \cite{vretblad2003fourier}, one obtains  
	\begin{align}
	\begin{split}
	&1)\;H^2_{1a}+H^2_{1b}+(I'_c(w)/m_c)>R\cdot|H_{1a}|, \\
	&2)\;H^2_{1a}+H^2_{1b}+ \left(I'_c(w)/m_c\right) > \sqrt{2}R \cdot\left(|H_{1a}|+|H_{1b}|\right)/2,\\ 
	&3)\;H^2_{1a}+H^2_{1b}+(I'_c(w)/m_c)>R\cdot|H_{1b}|,
	\end{split}
	\label{Eq:Inequalitiesfourirerpart}
	\end{align}
	where $I'_c(w)$ is the Fourier Transform of the inertia tensor of $I_c$, while these terms become 
	\begin{align}
	\begin{split}
	&\;H^2_{1a}+H^2_{1b}= \pi (a^2 n^2+a^2+2r^2) \delta(w) \\
	&+(\pi a^2 (n^2-1)/2) \cdot \left[ e^{-2j \varepsilon }\delta(2n+w)+e^{2j \varepsilon } \delta(2n-w)  \right] \\
	&+2 \pi a  r  \big[ e^{-j \varepsilon }\delta(n+w)-e^{j \varepsilon } \delta(n-w)  \big]j,\\
	&\;|H_{1a}|= \pi a n \big[ e^{-j \varepsilon }\delta(n+w)+e^{j \varepsilon } \delta(n-w)  \big] ,\\
	&\;|H_{1b}|= 2\pi r\delta(w) + \pi a  \big[ e^{-j \varepsilon }\delta(n+w)-e^{j \varepsilon } \delta(n-w)  \big] j.
	\label{Eq:CondtionFourierTransform}
	\end{split}
	\end{align}
By using transformed equation (\ref{Eq:CondtionFourierTransform}), the terms are simplified to two base waves for comparison: the first term is the constant shift by $\delta(w)$ and the second is the sinusoidal waves, $\delta(n+w)+\delta(n-w)$. Because the angular rotation $\zeta(\gamma,\theta)$ of the waves in both sides of inequality is always same, each side of (\ref{Eq:Inequalitiesfourirerpart}) can be compared relative to its multiplier $\delta$ with the same frequency $w$. By the known insight in the expressed property, all three conditions in (\ref{Eq:Inequalitiesfourirerpart}) are collected for each sinusoidal impulses $\delta$ in the given frequency $w$. 
	\begin{equation} {\small
	\begin{split}
	&1)\begin{cases}
	& 2rj (e^{-j \varepsilon }\delta(n+w)-e^{j \varepsilon } \delta(n-w)) > Rn \big [ e^{-j \varepsilon }  \\
	&\delta(n+w)+e^{j \varepsilon } \delta(n-w) \big ]
	\end{cases}\\
	&2)\begin{cases}
	& 4rj (e^{-j \varepsilon }\delta(n+w)-e^{j \varepsilon } \delta(n-w)) >   \sqrt{2}R \big [ e^{-j \varepsilon }
	\\
	&\delta(n+w) (n+j)+e^{j \varepsilon } \delta(n-w)(n-j) \big ]
	\end{cases}\\
	&3)\begin{cases}
	& 2rj (e^{-j \varepsilon }\delta(n+w)-e^{j \varepsilon } \delta(n-w)) > Rj \big [e^{-j \varepsilon }\delta(n \\
	&+w)-e^{j \varepsilon } \delta(n-w) \big ]
	\end{cases}
	\end{split}}
	\label{Eq:Amplitude3TranssWaves} 
	\end{equation} 
	Now, the minimum shift $\Delta \mu$ for having the left-hand side larger than the right-hand has to be computed by taking the Inverse Fourier Transform from (\ref{Eq:Amplitude3TranssWaves}). Thus, the inverse Fourier Transform of (\ref{Eq:Amplitude3TranssWaves}) is calculated by (\ref{Eq:FourierEquation}) for each case
	\begin{equation}{\small
		\begin{split}
		&1) \; 2ra\sin(n \zeta_1 +\varepsilon) > R a n\cos(n \zeta_1 +\varepsilon )\\
		&2)\;  4ra \sin(n \zeta_2 +\varepsilon) > \sqrt{2}R \Big [a\sin(n \zeta_2 +\varepsilon)+a n\cos(n \zeta_2 +\varepsilon) \Big ]\\
		&3)\;2ra \sin(n \zeta_3 +\varepsilon) > Ra \sin( \zeta_3 +\varepsilon)
		\end{split}}
	\label{Eq:Amplitude3TranssWavestimedomain}
	\end{equation}
	
	By checking (\ref{Eq:Amplitude3TranssWavestimedomain}), we see that there are $\pi/2$ and $\pi/4$ phase differences between each side of sinusoidal curves in conditions 1 and 2, respectively. To find the required shift for first two conditions, the derivative of time-domain forms in (\ref{Eq:Amplitude3TranssWavestimedomain}) are derived
	\begin{align*}
	\begin{split}
	&1) \frac{d}{d\zeta_1}\left[\frac{Rn \cos(n \zeta_1+\varepsilon)}{2r\sin(n \zeta_1+\varepsilon)}\right]= \frac{-Rn \sin(n \zeta_1+\varepsilon)}{2r\cos(n \zeta_1+\varepsilon )}=1,\\
	&2) \frac{d}{d\zeta_2}\left[\frac{\sqrt{2}R[n \cos(n \zeta_2+\varepsilon )+\sin(n \zeta_2+\varepsilon)]}{4r\sin(n\zeta_2+\varepsilon)}\right]\\
	&=\frac{\sqrt{2}R[-n\sin(n\zeta_2+\varepsilon)+\cos(n\zeta_2+\varepsilon)]}{4r\cos(n\zeta_2+\varepsilon)}=1.
	\end{split}
	\end{align*}
	Then, we solve it for $\gamma_1=n \zeta_1+\varepsilon$ and $\gamma_2=n \zeta_2+\varepsilon$, and find the tangential point of two waves when their slopes are the same 
	\begin{equation}
	\begin{split}
	&\gamma_1=  \tan^{-1}(-2r/Rn),\\ &\gamma_2=  \tan^{-1}\left(\left(R-2\sqrt{2}r\right)/R n\right).
	\end{split}
	\label{Eq:Gamma3anglecondition}
	\end{equation}
	We define $\Delta \mu_1$ and $\Delta \mu_2$ as the minimum required shifts for the right-hand side of inequalities to be always larger than left in conditions 1 and 2 by equaling both sides of (\ref{Eq:Amplitude3TranssWavestimedomain}) as
	\begin{align}
	\begin{split}
	&\Delta \mu_1+ R a n \cos(\gamma_1)=2ra \sin(\gamma_1),\\
	&\Delta \mu_2+ \frac{\sqrt{2}a }{2}R \left [ \sin(\gamma_2)+a n \cos(\gamma_2) \right]=2ra \sin(\gamma_2).
	\label{Eq:delta1Equation}
	\end{split}
	\end{align}
	The third condition in (\ref{Eq:Amplitude3TranssWavestimedomain}) is easy to compute because waves of both sides are in the same phase, hence,  $\Delta \mu_3$ is obtained by maximum amplitude difference  
	\begin{align}
	\Delta \mu_3=a \cdot |2r-R|.
	\label{Eq:delta3Equation}
	\end{align}
	Finally, substituting (\ref{Eq:Gamma3anglecondition}) into (\ref{Eq:delta1Equation}) and reordering them with respect to the minimum shifts $\Delta \mu_i$ results in 
	\begin{equation}{\footnotesize
	\begin{split}
	&1) \Delta \mu_1=a \cdot \Bigg|2r\sin\left(\tan^{-1}\left(\frac{-2r}{Rn}\right)\right)-Rn\cos\left(\tan^{-1}\left(\frac{-2r}{Rn}\right)\right)\Bigg|\\
	&2) \Delta \mu_2=\frac{\sqrt{2}a}{2}\bigg|(2\sqrt{2}r-R)\sin\left(\tan^{-1}\left(\left(R-2\sqrt{2}r\right)/Rn\right)\right)\\
	&-Rn\cos\left(\tan^{-1}\left(\left(R-2\sqrt{2}r\right)/Rn_i\right)\right)\bigg|\\
	&3) \Delta \mu_3=a\cdot |2r-R|
	\end{split}}
	\label{Eq:Comparesinespi}
	\end{equation}
    By putting Eq. (\ref{Eq:Comparesinespi}) back into inequality (\ref{Eq:Inequalitiesfourirerpart}) and taking the Inverse Fourier for single wave with $2n$ frequency and constant shift for $\delta(w)$, we have
    	\begin{align}
    \begin{split} 
    & 1)\; \frac{a^2 (n^2+1) }{2}+ \frac{a^2 (n^2-1) }{2}\cos(2(n \zeta_1 +\varepsilon))+r^2\\ 
    &+ \frac{I_c}{m_c} > \Delta \mu_{1}(a,n), \\
    &2)\; \frac{a^2 (n^2+1) }{2}+ \frac{a^2 (n^2-1) }{2}\cos(2(n \zeta_2 +\varepsilon))+r^2\\ 
    &+ \frac{I_c}{m_c} > \frac{\sqrt{2}}{2}Rr+\Delta \mu_{2}(a,n), \\
    &3)\; \frac{a^2 (n^2+1) }{2}+ \frac{a^2 (n^2-1) }{2}\cos(2(n \zeta_3 +\varepsilon))+r^2\\ 
    &+ \frac{I_c}{m_c} > Rr+\Delta \mu_{3}(a,n), 
    \end{split}
    \label{Eq:ConditionsForsingularitDesignBefore}
    \end{align}	
    To simplify the second term at right-hand side of inequalities (\ref{Eq:ConditionsForsingularitDesignBefore}), we choose $n>2$ which transfer $\cos(2(n \zeta +\varepsilon))$ to $\cos{2\varepsilon}$ in all conditions $\{\zeta_1,\zeta_2,\zeta_3\}$. Under the given assumption ($n>2$), Eq. (\ref{Eq:ConditionsForsingularitDesign}) can be derived from (\ref{Eq:ConditionsForsingularitDesignBefore}). 
\end{proof}
\begin{rem}
    Because the inertia tensor $I_c$ of the rotating mass is normally related to geometric objects (connecting cylindrical bar of pendulum) with a constant radius, it has been included as the constant value to the inequality.  
    \end{rem}
\begin{rem}    This Theory \ref{prop:singularfreeparametersd} can easily be extended for any underactuated system with two-link manipulators (for example the Acrobat) since $M_{12}$ term is in common in all models and does not have the inertia tensor of carrier $I_b$.
\end{rem}

In this study, we choose a 4th order Beta function \cite{svinin2015dynamic,TAfrishiRussiCha2019} to arrive the carrier $
\theta(t)$ toward its desired final configurations $\theta_{des}$ by
\begin{equation}
\begin{split}
\theta(t)= k \left( -\frac{20}{T^7}t^7+\frac{70}{T^6}t^6-\frac{84}{T^5}t^5+\frac{35}{T^4}t^4 \right),
\end{split}
\label{Eq:BetaFunctions}
\end{equation}
where $T$ and $k$ are the time constant of designed motion and the value for the final arrived distance $\theta(T)=k=\theta_{\textrm{des}}$. We expect from this feed-forward control to actuate the rotating mass like $\gamma(t)=d^2\theta(t) /dt^2$ from (\ref{Eq:BetaFunctions}) as a two-step motion. This two-step motion of rotating mass [see $\gamma+\theta$ at Fig. \ref{Fig:MassPointSimu2}-b as an example of this motion pattern] is followed by a counterclockwise rotation till certain angle $\gamma_{max}$ and a similar clockwise rotation for returning to the rest position. Note that similar to what has been developed in
\cite{svinin2015dynamic,TAfrishiRussiCha2019}, one can show that with the selection of this motion scenario the condition $\dot\theta(T)=0$ is always satisfied.
\section{SIMULATION ANALYSIS}
In this section, the proposed theory is analyzed in the simulation space. At first, to evaluate our model in the worst-case scenario, a mass-point system with $I_c=0$ is chosen. We find the singular-free model with satisfying the conditions of the proposed theorem. Next, to compare the modified model with the classic model, we compare both cases when there is an inertia tensor $I_c$ as a pendulum system. 

\begin{table}[b!]
	\renewcommand{\arraystretch}{1}
	\caption{Value of parameters for the simulation studies.}
	\label{Tab:DesignParameter}
	\centering
		\begin{tabular}{cccc}
			\hline
			Variable & Value & Variable & Value\\
			\hline
			$m_c$ & $0.4\;$kg&$r$ &  $0.131 \;$m\\
			$M_b$ &$ 1 \; $ kg & $R$  &  $0.145\;$ m\\
			$g$&$9.8 \;\; $m/s$^2$& $I_b$&$0.0140\;$kg$\cdot$m$^2$\\
			\hline
	\end{tabular}
\end{table}
%\begin{figure}[t!]
%	\centering
%	\includegraphics[width=1.8 in, height= 1.5 in]{Mass_Point_CircleCurve1.pdf}	
%	\caption{Geometric shape of combined waves on the trajectory for the rotating mass. Note that the combined waves resultant trajectory always align with ideal rotating mass circle %$r$.}\label{Fig:Mass_Point_CircleCurve1}
%\end{figure}
\begin{figure}[t!]
	\centering
	%a)	\includegraphics[width=2.6 in, height= 1.6 in]{Singular_Tra_Plot1.pdf}\\
	\includegraphics[width=2.5 in,height=1.4 in]{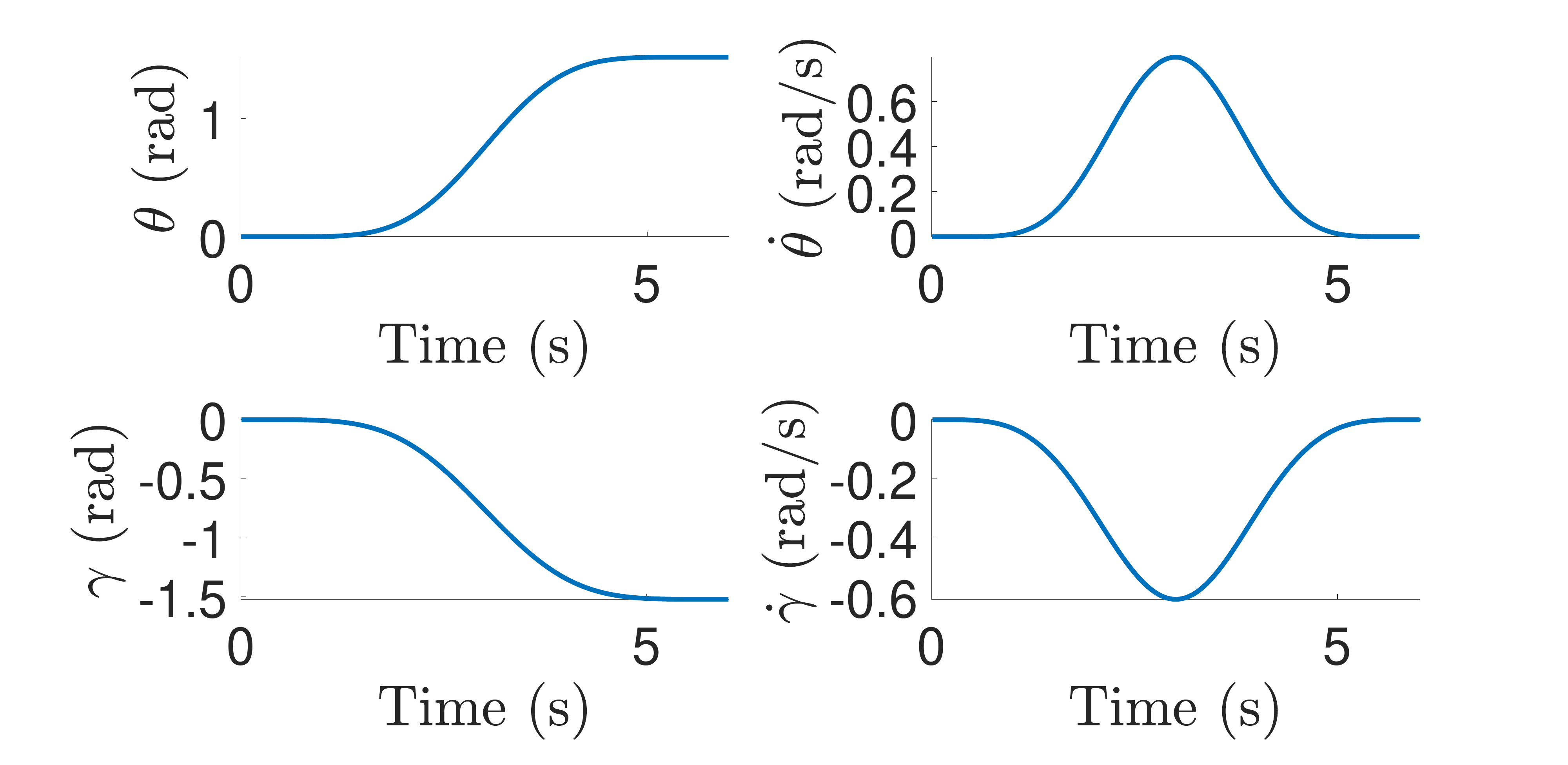}	
	\caption{Example simulation for passive carrier $\{\theta,\dot\theta\}$ and rotating mass-point $\{\gamma,\dot\gamma\}$ states by the modified model.}\label{Fig:TestSingularity1}\label{Fig:MassPointSimu1}
\end{figure}

Matlab ODE45 is used during the simulation studies of derived model for solving (\ref{Eq:termsofmusindirect}) and (\ref{Eq:FinalEqInverse})-(\ref{Eq:ExtraFinalTerms}) equations. The accuracy of the solver for relative and absolute errors are 0.0001 and
0.0001, respectively. The geometric parameters of the considered physical system are like Table \ref{Tab:DesignParameter} \cite{TafrishiASME2019}. Note that this system fails the singular-free coupling condition (\ref{Eq:distancemasspointexample}) by having $r<R$, as shown in Example 1, which makes the system to be in a singular region. To prescribe the angular orientation of the spherical carrier the introduced Beta functions in (\ref{Eq:BetaFunctions}) is applied where $k=\theta_{des}$ is $\pi/2$ rad. The simulation is run for $6$ s with $T=6$. The robot begins from rest condition and it is expected to reach rest position at end of simulation time by the prescribed Beta function.

%\begin{figure}[t!]
%	\centering
% 	\includegraphics[width=2.4 in]{Singular_Tra_Plot2.pdf}\\
%	b)	\includegraphics[width=2.4 in]{Singular_Tra_Plot3.pdf}	
%	\caption{ Inertia term and output torque results. Note that $1/M_{12}$ is taken from internal solution of differential solvers. }\label{Fig:TestSingularity2}
%\end{figure}
%To have classic model with circular $r$ radius trajectory without any combined wavy curves, the terms in (\ref{Eq:SinsiodalTotal})-(\ref{Eq:VelocitySinsiodalTotal}) are chosen as $a_i=n_i=\varepsilon_i=0$ like expressed in Example 1. Fig. \ref{Fig:TestSingularity1}-a shows the simulation results for the rotating mass $\{\gamma,\dot\gamma\}$ with specified $\{\theta,\dot\theta,\ddot\theta\}$ spherical carrier orientation. It is vivid that solver breaks and singularity converge system to infinity. Also, the control input $\tau_{\gamma}$ on the rotating mass as Fig. \ref{Fig:TestSingularity2} follows the same trend. We can see that $M_{12}$ fails the positive definiteness condition during the simulation of this underactuated system, $\overline{M}<0$.

\begin{figure}[t!]
	\centering
	a)	\includegraphics[width=2.3 in,height=.9 in]{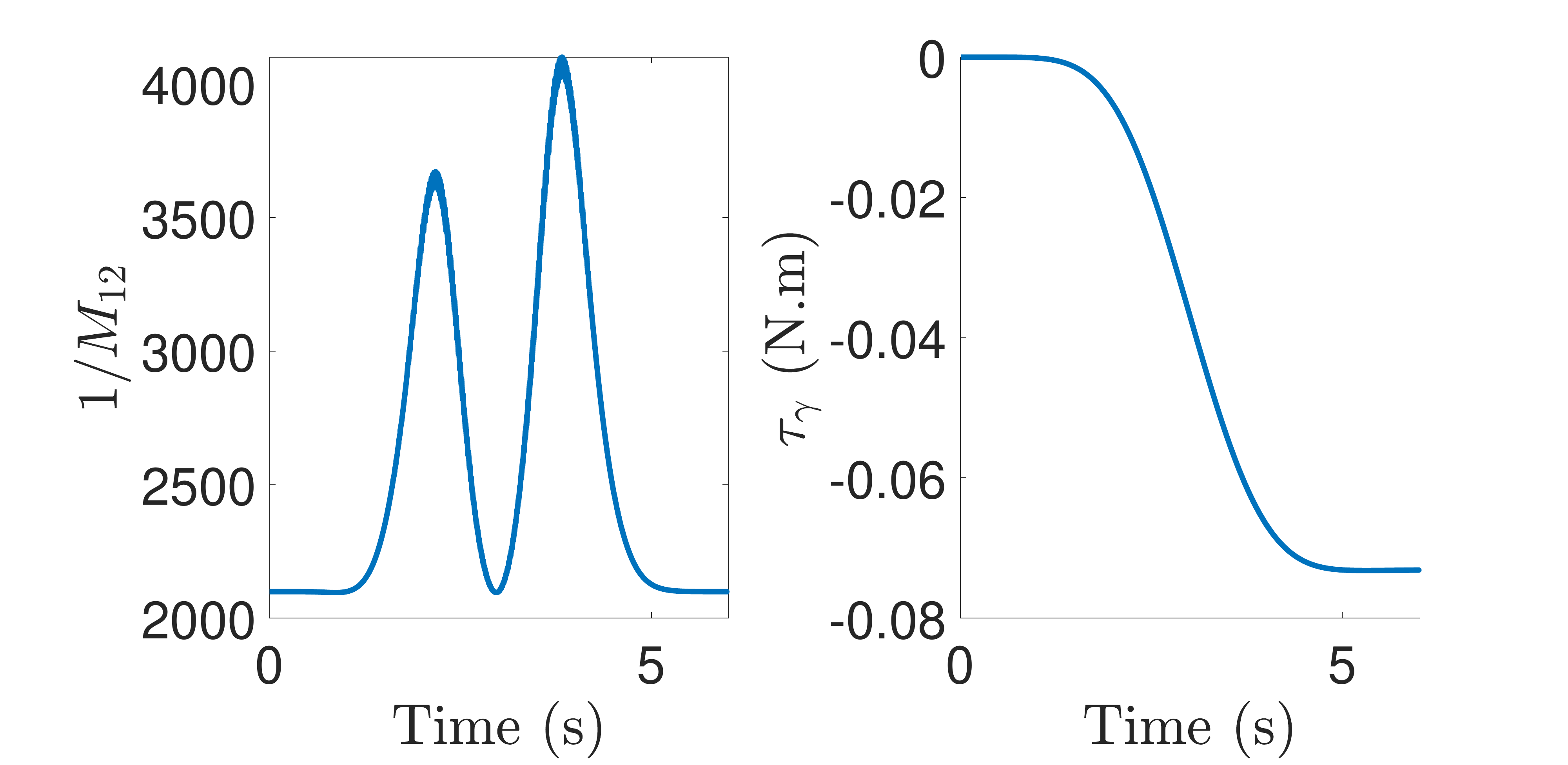}\\
	b)	\includegraphics[width=2.3 in,height=.9 in]{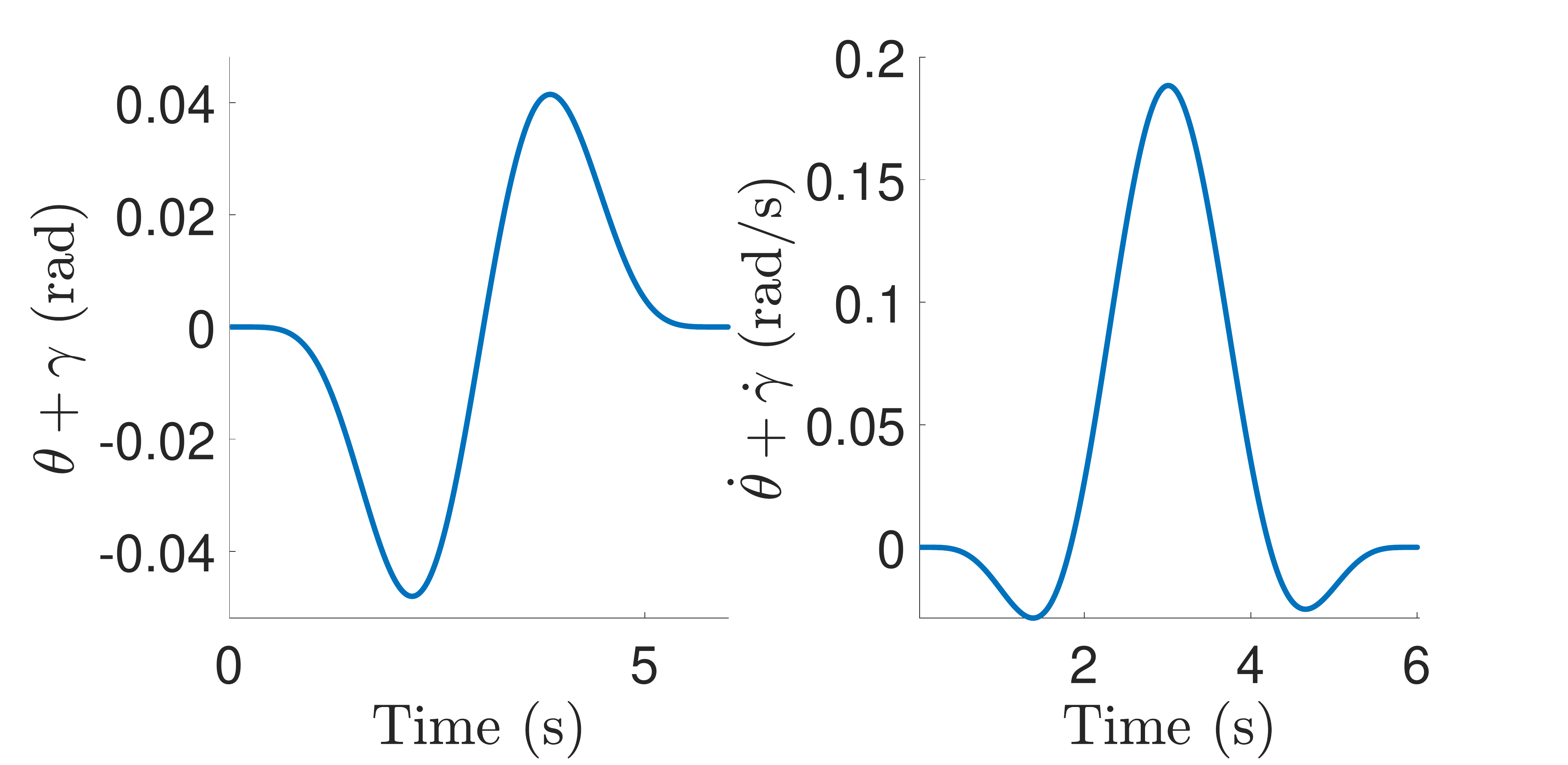}	
	\caption{a) Inertia term and output torque results for modified inverse dynamics, b) the true location and velocity of the rotating mass respect to reference frame.}\label{Fig:MassPointSimu2}
\end{figure}

To solve the singularity, we utilize the proposed Theory \ref{prop:singularfreeparametersd} where $a$, $n$ and $\varepsilon$ are designed under the condition of (\ref{Eq:ConditionsForsingularitDesign}). By substituting the values of the geometric parameters from Table into the conditions (\ref{Eq:ConditionsForsingularitDesign}), we choose our wave parameters with first maximum value as $a= 0.0055$, $n=10$ and $\varepsilon=0$ which satisfy all three inequalities. Note, if we have an inertia tensor $I_c$, depending on the designed mechanism, wave amplitude $a$ can be chosen smaller as indicated in inequalities (\ref{Eq:ConditionsForsingularitDesign}).

By running the simulation with obtained parameters, we see that the inverse dynamics are integrated without hitting any singularity [See Fig. \ref{Fig:MassPointSimu1}-b and \ref{Fig:MassPointSimu2}]. 
As expected, the rotating mass follows a smooth two-phase motion with the applied feed-forward control by the Beta function [see Fig. \ref{Fig:MassPointSimu2}-b]. Also, the control torque $\tau_{\gamma}$ as the output is produced responsively with solving the modified nonlinear dynamics which displaces the spherical carrier to the desired location with the rest-to-rest motion.  
 \begin{figure}[t!]
 	\centering
 	\includegraphics[width=2.2 in,height=.9 in]{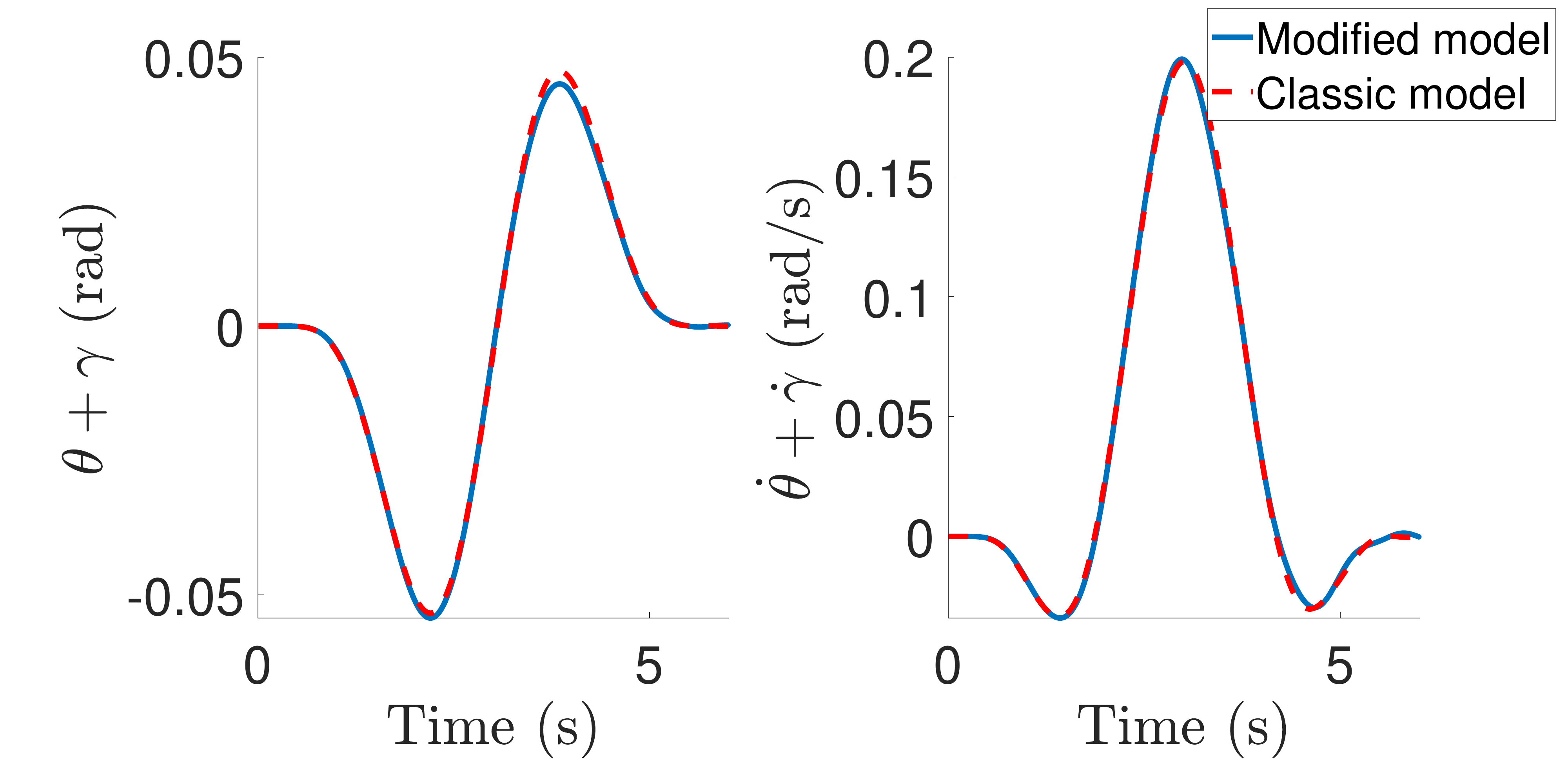}	
 	\caption{Compared results for a classic and modified motion equations in a pendulum system with the inertia tensor $I_c$.}\label{Fig:CompareClassicModi}
 \end{figure}

Finally, we show a comparison between a classic rotating pendulum system with our modified singular-free model. We do this comparison for the sake of clarification that model with small-amplitude combined wave doesn't hurt/diverge the motion equations while it is removing the singularities due to coupling. All the simulation parameters are similar to the previous case-study except we include an inertia tensor of cylindrical pendulum $I_c=m_lr^2/3=0.0057$ kg$\cdot$m$^2$ that is connected to the mass-point as Fig. \ref{Fig:wavy2}. Also, we simulate the classic model of a pendulum system for the rolling sphere from Refs. \cite{kayacan2012modeling,svinin2015dynamic} which same model can be derived by specifying $a=n=\varepsilon=0$ at Eq. (\ref{Eq:termsofmusindirect}). Fig. \ref{Fig:CompareClassicModi} shows that our model doesn't have any dissimilarity with the classic model with included wave on the circular trajectory, $a \ll r,R$. In this case, note that inclusion of $I_c$ variable satisfies the singular-free coupling condition in Eq. (\ref{Eq:SingularFreePendl}). Finally, our designed theorem can easily work for the geometries that singular-free coupling conditions was limited by (\ref{Eq:Theconditionmass-pintsin}) and (\ref{Eq:SingularFreePendl}) (worst-case as a mass-point) conditions in previous studies.  

\section{CONCLUSIONS}

In this paper, we proposed that designing the trajectory of the rotating mass via the combined sine wave omits the singularities in the underactuated systems. We started by introducing the kinematics of the small-amplitude wave on the circular rotation of a mass-point. Next, the modified nonlinear dynamics of the underactuated rolling system were derived by Lagrangian equations. Before developing the condition for singularity-free inverse dynamics, the singularity regions relative to the geometric parametrization are shown for the spherical rolling systems. In the end, the solution of the developed theorem is demonstrated with example simulations where the states of the rolling carrier are specified by the Beta function.

As the advantage of our modified model, the solution is free from any complex algorithm or space transformation. This can resolve the limitations of the physical mechanism design due to the inertial coupling in the underactuated systems with 2 degrees of freedom. Also, it facilitates the applications of different advanced feedback controllers without any limited configurations. In the future, we plan to extend this theory for holonomic mechanisms with multiple degrees-of-freedom and we will check how can designing each passive constraint with various phase-shifted waves would prevent the integrated singular configurations.

\section*{Acknowledgment}
This research was supported, in part, by the Japan Science and Technology Agency, the JST Strategic International Collaborative Research Program, Project No. 18065977.

\bibliographystyle{IEEEtran}
\bibliography{NonlinearSingularityIEEEShort}

%%%%%%%%%%%%%%%%%%%%%%%%%%%%%%%%%%%%%%%%%%%%%%%%%%%%%%%%%%%%%%%%%%%%%%%%%%%%%%%%

%%%%%%%%%%%%%%%%%%%%%%%%%%%%%%%%%%%%%%%%%%%%%%%%%%%%%%%%%%%%%%%%%%%%%%%%%%%%%%%%

\addtolength{\textheight}{-12cm}   % This command serves to balance the column lengths

\end{document}